\documentclass[11pt]{article}
\usepackage[margin=1in]{geometry}
\pdfoutput=1
\usepackage{ifthen}
\newboolean{usenatbib}
 
\setboolean{usenatbib}{false}
 
\usepackage[style=alphabetic,maxalphanames=4,minalphanames=3,maxcitenames=2,maxbibnames=99,giveninits=false,doi=false,url=false,natbib]{biblatex}
\DeclareDelimFormat[bib]{nametitledelim}{\addspace}
 
\usepackage{bm}
\usepackage{fullpage}
\usepackage[english]{babel}
\usepackage[utf8]{inputenc}
\usepackage[T1]{fontenc}
\usepackage{ae} \usepackage{aecompl}
\usepackage{enumitem}
\setlist{nosep}
\usepackage{microtype}
\usepackage{booktabs}
\usepackage{hhline}
\usepackage{makecell}
\usepackage{etoolbox}
\apptocmd{\sloppy}{\hbadness 10000\relax}{}{}
 
\usepackage{amsmath,amsthm,amssymb,mathtools,thmtools}
\usepackage{graphicx}
\usepackage{color}
\usepackage[dvipsnames]{xcolor}
\usepackage[textsize=scriptsize,disable]{todonotes}
\usepackage[colorlinks=true, allcolors=magneta]{hyperref}
\usepackage[algo2e,ruled]{algorithm2e}
\usepackage{xfrac}
\usepackage[normalem]{ulem}

\hypersetup{colorlinks,
            breaklinks=true,
            linkcolor=purple,
            citecolor=purple,
            urlcolor=magenta,
            linktocpage,
            plainpages=false,
            bookmarks=false}
 
\usepackage{macros}
 
\usepackage{comment}
\usepackage{tabularx}
\usepackage{subfig}
\usepackage{array}

\title{Robust Online Learning}
\usepackage{xcolor}
\usepackage{tikz}
\usetikzlibrary{trees}

\author{%
 Sajad Ashkezari\thanks{University of Waterloo, \texttt{sajad.ashkezari@uwaterloo.ca}.}
}


\begin{document}

\maketitle

\begin{abstract}%
  We study the problem of learning robust classifiers where the classifier will receive a perturbed input. Unlike robust PAC learning studied in prior work, here the clean data and its label are also adversarially chosen. We formulate this setting as an online learning problem and consider both the realizable and agnostic learnability of hypothesis classes. We define a new dimension of classes and show it controls the mistake bounds in the realizable setting and the regret bounds in the agnostic setting. In contrast to the dimension that characterizes learnability in the PAC setting, our dimension is rather simple and resembles the Littlestone dimension. We generalize our dimension to multiclass hypothesis classes and prove similar results in the realizable case. Finally, we study the case where the learner does not know the set of allowed perturbations for each point and only has some prior on them.
\end{abstract}


\section{Introduction}
In this paper we study the online learning of robust predictors, i.e., predictors whose prediction remains correct even if an adversary perturbs their input. It has been shown that even predictors that have high accuracy on clean data can drastically fail on slightly perturbed inputs, which could be indistinguishable from the clean input for humans \citep{goodfellow2014explaining}. Learning robust neural networks has been widely studied and is an ongoing research direction \citep{chakraborty2021survey}.

Our goal is to formulate and study this problem through the theoretic online learning framework \citep{littlestone1988learning}. Informally, we consider the following interactive learning game between an adversary and a learner. At each round the adversary reveals a perturbed input to the learner who will then predict a label on this input. The adversary then reveals the clean input and its true label. The goal of the learner is to minimize the number of mistakes it makes. We formally define this problem in Definition~\ref{def: robust_online_learning}. We will then define a notion of robust online learnability of a class and answer the following questions:

\begin{itemize}
    \item What is the optimal number of mistakes achievable in the robust online learning problem?
    \item Is there a combinatorial measure of complexity of a class that controls this optimal value?
\end{itemize}

\subsection*{Related work}

Robust learnability has been extensively studied in the learning theory literature. However, most prior work focus on robust PAC learnability where unlike our setup, the clean data comes from a distribution and is then manipulated. It has been shown that VC classes are robustly PAC learnable, but finite VC is not a necessary condition \citep{pmlr-v99-montasser19a}. It was later shown that the learnability is characterized by a new dimension that depends on the \textit{global one-inclusion graph} of a hypothesis class \citep{MontasserHS22}. Other works have also studied robust PAC learnability when the perturbation set that the adversary is allowed to perturb a point into is unknown to the learner or belongs to a class of such sets \citep{montasser2021adversarially, lechner2023adversarially}. Adversarially robust PAC regression has been studied by \citet{pmlr-v202-attias23a}.

The online learning framework was introduced in the seminal work of Littlestone \cite{littlestone1988learning}, which was later extended to agnostic online learning as well \cite{Ben-DavidPS09}. This framework has also been studied for multiclass hypothesis classes \cite{daniely2015multiclass, hanneke2023multiclass}. To the best of our knowledge, this is the first work to study robust online learning in this framework.

\subsection*{Contributions}
We make the following contributions:

\begin{itemize}
    \item We formulate the robust online learning problem through a well-known theoretical framework.
    \item We define a new dimension of a binary class $\H$, $\ldimu(\H)$, and show that it characterizes robust online learnability. We show  the optimal mistake bound in the realizable case is exactly equal to this dimension. In the agnostic case, we show the optimal expected regret up to logarithmic factors is $\tilde{\mathcal{O}}(\sqrt{\ldimu(\H)T)}$.
    \item We extend our dimension for multiclass hypothesis classes and show similar results for realizable robust online learning.
    \item We also study the setting where the learning algorithm does not exactly know the perturbations that the adversary is allowed to make, but has prior knowledge that it belongs to a finite family $\mathcal{G}$ of perturbation functions. We prove upper bounds for this setting that depend logarithmically on the cardinality of $\mathcal{G}$.
\end{itemize}

\section{Setup and Problem Formulation}
We use $\X$ and $\Y$ to denote our instance space and label space, respectively. Here, we focus on binary label space, i.e., $\Y=\{0, 1\}$. A hypothesis is a mapping $h:\X\to\Y$ from instance space to label space. A hypothesis class, $\H \subseteq \Y^\X$, is a set of such mappings. A learning algorithm $\A: (\X \times \Y)^* \to \Y^\X$ gets a finite number of instance-label pairs and outputs a hypothesis. We let $\U:\X \to 2^\X$ be any function that maps instances to the set of allowed perturbations. That is, for each $x\in\X$, the adversary can perturb it only to $z\in \U(x)$. Unless stated otherwise, we assume $\U$ is fixed and known to the learner (we study uncertain perturbation sets in Section~\ref{sec: uncertain perturbations}). The adversarial loss of a hypothesis $h$ on $(x,y)$ is defined as follows:

\begin{equation}
    l_\U(h, (x,y)) = \sup_{z \in \U(x)}\indict[h(z) \neq y],
\end{equation}

where $\indict[A]=1$ if $A$ is true, and $\indict[A]=0$ otherwise.

We are now ready to define our learning problem.

\begin{definition}[Robust Online Learning]
    \label{def: robust_online_learning}
    Robust online learning is an iterative game between an adversary and a learner such that:
    
    At each round $t=1,2,\cdots$:
    \begin{enumerate}
        \item[] 
        \begin{enumerate}
            \item The adversary selects $Z_t$ and reveals it to the learner.
            \item The learner predicts $\yhat_t \in \Y$.
            \item The adversary selects $X_t\in\X$ with $Z_t\in\U(X_t)$ and $Y_t \in \Y$ and reveals $(X_t, Y_t)$ to the learner.
            \item The learner incurs a loss of $\indict[\yhat_t \neq Y_t]$.
        \end{enumerate}
    \end{enumerate}
\end{definition}

\begin{remark}
    In Definition \ref{def: robust_online_learning}, we assume the adversary knows how the learner will predict on each point and the goal of the adversary is to maximize the number of mistakes. So we can equivalently reformulate the problem such that at each round the learner first picks a hypothesis $h_t:\X \to \Y$ and then the adversary picks and reveals $(\X_t, \Y_t)$. Finally, the learner incurs a loss of $l_\U(h, (X_t,Y_t))$.
\end{remark}

Consider a sequence $Z_1, X_1,Y_1, \cdots, Z_T, X_T, Y_T$ for $T \leq \infty$. For any finite $t \leq T$, define $S_{<t}=Z_1,X_1, Y_1, \cdots, Z_{t-1},X_{t-1}, Y_{t-1}$. We respectively define the mistake bound and the regret of a learner $\A$ with respect to $\H$ on a sequence $S$ as follows:

\begin{equation}
    \mistake(\A, S) = \sum_{t=1}^{T} \indict[\A(S_{<t})(Z_t) \neq Y_t]
\end{equation}

\begin{equation}
            \label{def: reg}
            R_\H(\A, S) = \sum_{t=1}^T\indict[\A(S_{<t})(Z_t)\neq Y_t] - \min_{h\in \H} \sum_{t=1}^T l_\U(h, (X_t,Y_t)) 
\end{equation}

Similar to the classic online learning, we can define robust online learning with respect to a hypothesis class. We say a sequence $X_1,Y_1,\cdots,X_T, Y_T$ is $\U$-robust realizable with respect to $\H$ if $\inf_{h \in \H} \sum_{t=1}^{T} l_\U(h, (X_t,Y_t)) = 0$. Let $\realizableset(\H)$ denote the set of all realizable sequences w.r.t. $\H$.

\begin{definition}[Realizable Robust Online Learnability]
    We say that $\H$ is realizable robust online learnable with optimal mistake bound $\optmistake < \infty$, if the following holds:
    \begin{equation}
        \inf_{\A} \sup_{S\in \realizableset(\H)} \mistake(\A, S) = \optmistake
    \end{equation}
\end{definition}

We can also define a learnability task without the realizability assumption, which is known as agnostic learning. However, instead of bounding the number of mistakes, we are interested in bounds on the regret. Furthermore, in this setup, we assume the number of rounds is some finite $T$ known to the learner in advance.

\begin{definition}[Agnostic Robust Online Learnability]
    We say that $\H$ is agnostic robust online learnable with optimal regret $\optregret_T < \infty$ for any $T < \infty$ if the following holds:
    \begin{equation}
        \inf_{\A} \sup_{S:|S|=T} \regret_\H(\A, S) = \optregret_T
    \end{equation}
\end{definition}

\section{Optimal mistake bounds for realizable robust online learning}
\label{sec: realizable binary}
In this section we derive optimal mistake bounds for realizable robust online learning of a class $\H$ with respect to a combinatorial parameter of it. To do this, we first introduce an easier version of the robust online learning problem, which makes the definition of our dimension more intuitive.

\subsection{Orientation Game}
For any given $Z_t$, there could be many candidates $x\in \X$ such that $Z_t\in \U(x)$. So the decision of the learner will potentially need to depend on all these points and how the class behaves on them. What if we are given only two candidates $X_t^0$ and $X_t^1$ such that $Z_t \in X_t^0 \cap X_t^1$ and only need to make a decision between these two? We will show that unlike our original problem, finding a learner for this problem can be simpler. We first formalize this problem.

\begin{definition}[Orientation Game]
    \label{def: orientation game}
    Orientation game is an iterative game between an adversary and a learner such that:
    
    At each round $t=1,2,\cdots$:
    \begin{enumerate}
        \item[] 
        \begin{enumerate}
            \item The adversary selects $X_t^0, X_t^1 \in \X$ with $\U(X_t^0) \cap \U(X_t^1) \neq \emptyset$ and reveals them to the learner.
            \item The learner predicts $\yhat_t \in \{0,1\}$.
            \item The adversary selects $Y_t \in \Y$ and reveals $(X_t^{Y_t}, Y_t)$ to the learner.
            \item The learner incurs a loss of $\indict[\yhat_t \neq Y_t]$.
        \end{enumerate}
    \end{enumerate}
\end{definition}

We say that a sequence $S=(X_1^0,X_1^1), Y_1, (X_2^0,X_2^1), Y_2, \cdots, (X_T^0,X_T^1), Y_T$ is realizable by $\H$ if the sequence $X_1^{Y_1},Y_1, \cdots, X_T^{Y_T}, Y_T$ is realizable by $\H$. We also define $\X^2_\U:= \{(x_1,x_2)\in \X^2: \U(x_1) \cap \U(x_2) \neq \emptyset\}$. With these definitions, we are now ready to define our dimension.

\begin{definition}[$\U$-adversarial Littlestone tree] A $\U$-adversarial Littlestone tree is a full binary tree of depth $d\in \mathbb{N}$ whose internal nodes are labeled by $\X_\U^2$ and the two outgoing edges from each node are labeled by $0$ and $1$.
\end{definition}

In other words, the tree can be represented as the following collection:
$$\{(x^0_\mathbf{u}, x^1_\mathbf{u})\in \X_\U^2:\mathbf{u}\in\{0,1\}^k, 0\leq k < d\}$$

We say that a tree is shattered by $\H$ if each path emanating from the root is realizable by $\H$. That is, for each $\mathbf{u}=(u_1,\cdots, u_d)\in \{0,1\}^d$ and for each $0 \leq k < d$, there exists $h\in\H$ such that for all $0 \leq i \leq k$, $h(z)=u_{i+1}$ for all $z\in \U(x_{\mathbf{u}_{\leq i}}^{u_{i+1}})$, i.e., $l_\U(h, (x_{\mathbf{u}_{\leq i}}^{u_{i+1}}, u_{i+1}))=0$, where $\mathbf{u}_{\leq i}=(u_1,\cdots,u_i)$. Figure~\ref{fig:tree-example} illustrates an example of such a tree of depth 2.

\begin{definition}[$\U$-adversarial Littlestone Dimension]
    The $\U$-adversarial Littlestone dimension of a class $\H$ is the maximum depth of a tree that it shatters. We denote this dimension by $\ldimu(\H)$. Furthermore, we say $\ldimu(\H) = \infty$ if $\H$ shatters trees of arbitrary large depth. 
\end{definition}

\begin{figure}[t]
    \centering
    \begin{tikzpicture}[
        node/.style={circle, draw=black, thick, inner sep=2pt},
        edge/.style={draw=black, thick},
        level 1/.style={sibling distance=45mm},
        level 2/.style={sibling distance=30mm},
        level distance=15mm
    ]

    \node[node] (root) {$x_{\emptyset}^0,x_{\emptyset}^1$}
        child { node[node] (x0) {$x_{0}^0, x_0^1$}
            child { node[node] (x00) {$x_{00}^0,x_{00}^1$}
                edge from parent[edge] node[left] {0}
            }
            child { node[node] (x01) {$x_{01}^0, x_{01}^1$}
                edge from parent[edge] node[right] {1}
            }
            edge from parent[edge] node[left] {0}
        }
        child { node[node] (x1) {$x_{1}^0,x_1^1$}
            child { node[node] (x10) {$x_{10}^0, x_{01}^1$}
                edge from parent[edge] node[left] {0}
            }
            child { node[node] (x11) {$x_{11}^0, x_{11}^1$}
                edge from parent[edge] node[right] {1}
            }
            edge from parent[edge] node[right] {1}
        };

    \draw[very thick, red] (root) -- (x0);
    \draw[very thick, red] (x0) -- (x01);


    \end{tikzpicture}
    \caption{A $\mathcal{U}-$adversarial tree of depth 2. For each $k\leq 2$ and $\mathbf{u}\in \{0,1\}^k$ we have $\mathcal{U}(x_\mathbf{u}^0)\cap \mathcal{U}(x_\mathbf{u}^1)\neq \emptyset$. The tree is shattered by $\H$ if each of its root-to-leaf paths are realizable by $\H$. For example, there must exist $\textcolor{red}{h_{01}}\in \H$ such that $\textcolor{red}{h_{01}}(z)=\textcolor{blue}{0}$ for all $z\in \U(x_{\emptyset}^{\textcolor{blue}{0}})$ and $\textcolor{red}{h_{01}}(z)=\textcolor{purple}{1}$ for all $z\in \U(x_{0}^{\textcolor{purple}{1}})$.}
    \label{fig:tree-example}
\end{figure}
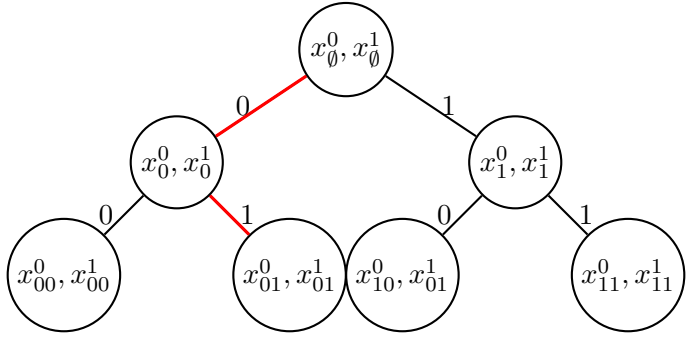

\begin{remark}
    A quick sanity check would be to see if our definition coincides with the classic Littlestone dimension when the adversary is not allowed to perturb the instances, i.e., $\U(x) = \{x\}$ for all $x\in\X$. In this case, $\X_\U^2 = \{(x,x):x\in \X\}$ as $\U(x_1) \cap \U(x_2) = \emptyset$ for all $x_1\neq x_2$. Thus, we can simply denote each node by an instance of $\X$ and the definition becomes the same as the classic definition.
\end{remark}

\begin{algorithm}
\caption{Standard Optimal Algorithm for Orientation Game ($\mathrm{SOA_{OG}}$)}
\label{alg: orientation learner}
\KwIn{Hypothesis class $\H\subseteq \Y^\X$, perturbation set $\U$}

Initialize $\mathcal{V}_0 \gets \H$\;

\For{$t \geq 1$}{
    Receive $(X_t^0, X_t^1)$\;
    
    \For{$y \in \Y$}{Define $\mathcal{V}_t^y = \{h\in \mathcal{V}_{t-1}:\forall z\in \U(X_t^y), h(z)=y\}$}
    
    Predict $\yhat_t$ for some= $\yhat_t \in \argmax_{y\in \Y} \ldimu(\mathcal{V}_t^y)$\;

    Receive $Y_t$\;

    Let $\mathcal{V}_{t}=\mathcal{V}_t^{Y_t}$
}
\end{algorithm}

Similar to the realizable robust online learning, we can also define realizable version of the orientation game and the optimal mistake. We denote the optimal mistake bound of this problem with $\optmistake_{OG}$. Then we get the following results on the optimal mistake. 

\begin{theorem}
    \label{thm: optimal mistake orientation game}
    A hypothesis class $\H$ with finite $\ldimu(\H)$ is realizable learnable in the orientation game with optimal mistake bound $\optmistake_{OG} = \ldimu(\H)$. In particular, the algorithm $\mathrm{SOA_{OG}}$ achieves the optimal mistake bound by ensuring that each mistake reduces the $L_\U$ dimension of the version space by at least one.
    
    \begin{proof} We first show that $\optmistake_{OG} \geq L:= \ldimu(\H)$. Fix any learning algorithm. The adversary plays according to a shattered tree of depth $L$, $\{(x^0_\mathbf{u}, x^1_\mathbf{u})\in \X_\U^2:\mathbf{u}\in\{0,1\}^k, 0\leq k < L\}$. It starts by picking the root of the tree, $(x_\emptyset^0,x_\emptyset^1)$, and for any prediction $\hat{y}$ of the algorithm, it chooses it's label to be $1-\hat{y}$ and then continue the process for the child in the tree connect to the current root and continuing the same process. By the definition of shattering, the adversary can continue this process for at least $L$ steps while maintaining realizability. Since we can force $L$ mistake on any learner, $\optmistake_{OG} \geq L$.
    
    We now show $\mathrm{SOA_{OG}}$ makes at most $L$ mistakes on any realizable sequence, which implies $\optmistake_{OG} \leq L$, and thus, $\optmistake_{OG} = L$. Consider any round $t$ at which the algorithm makes a mistake. We claim $\ldimu(\mathcal{V}_{t+1}) < \ldimu(\mathcal{V}_t) =: L_t$. Assume otherwise. Then it must be that $\ldimu(\mathcal{V}_{t+1}) = \ldimu(\mathcal{V}_t)$ as $\mathcal{V}_{t+1} \subseteq \mathcal{V}_{t}$ and thus $\ldimu(\mathcal{V}_{t+1}) \leq \ldimu(\mathcal{V}_t)$. By definition of $\yhat_t$, $\ldimu(\mathcal{V}_{t}) \geq \ldimu(\mathcal{V}_t^{\yhat_t}) \geq \ldimu(\mathcal{V}_t^{Y_t})=\ldimu(\mathcal{V}_{t+1})=\ldimu(\mathcal{V}_{t})$. Let $T_0$ and $T_1$ be trees of depth $L_t$ shattered by $\mathcal{V}_t^{0}$ and $\mathcal{V}_t^{1}$, respectively. Create a new tree whose root is $(X_t^0, X_t^1)$ and whose left and right subtrees are $T_0$ and $T_1$, respectively. They by definition of $\mathcal{V}_t^y$, this tree is shattered by $\mathcal{V}_t$, and thus $\ldimu(\mathcal{V}_t) \geq L_t + 1$ which is a contradiction and thus the claim holds. Thus, each time the algorithm makes a mistake, the dimension of $\mathcal{V}_t$ decreases by at least 1, since the dimension is nonnegative, the algorithm makes at most $L$ mistakes.
    \end{proof}
\end{theorem}

\subsection{From orientations to learners}

Here, we show how we can convert a learner in the orientation game to a learner for the robust online learning problem. Suppose the learner receives an input $z$ and wants to predict which 
$x$ with $z\in\U(x)$ the adversary will choose. For a candidate $x^y$ representing label $y$, the learner checks the decision of the orientation between this candidate and each candidate representing the opposite label. If the orientation is always towards $x^y$, then the learner predicts $y$. The learner does this process for all $y$ and all $x^y$. If the learner makes a mistake on $X_t, Y_t$ at some round $t$, then it must be that the orientation between $X_t$ and some candidate from the other label was wrong. Thus, we can upper bound the number of mistakes made in the online learning problem by the number of mistakes made in the orientation game, which we can upper bound by the results in the previous section. We present our learner in Algorithm \ref{alg: online learner} where we define $\mathcal{V}^\U_{x,y}:= \{h:\in \mathcal{V}: \forall z\in \U(x), h(z)=y\}$ be the set of hypotheses in $\mathcal{V}$ consistent on $(x,y)$. We note that this idea has also been used by \citet{MontasserHS22} for converting orientations of their global one-inclusion graph to PAC learners.

\begin{algorithm}[t]
\caption{Robust Online Learning Strategy $\yhat_t$}
\label{alg: online learner}
\KwIn{Hypothesis class $\H\subseteq \Y^\X$, perturbation set $\U$, orientation learner $\eta:(\X^2_\U \times \Y)^* \times \X^2\to \{0,1\}$, test point $z$}

Initialize $\mathcal{V} \gets \H$, $f(.,.) \gets \eta((.,.))$, $\tau \gets 1$\;

\For{$t \geq 1$}{
    Receive $Z_t$
    
    \For{$y \in \Y$}{
    Define $P_y := \{x\in\X : Z_t\in \U(x) \land \mathcal{V}^\U_{x,y} \neq \emptyset\}$\;
    }
    \If{$\exists y\in \Y, \exists x_y\in P_y, \forall x_{1-y}\in P_{1-y}: f(x_0,x_1)=y$}{
    $\yhat_t \gets y$\;
    }

    \Else{
    $\yhat_t \gets 1$\;
    }
    
    Receive $(X_t, Y_t)$\;
    
    \If{$\yhat_t \neq Y_t$}{
        Let $X_\tau^{Y_t} \gets X_t$\;
        
        Let $X_\tau^{1-Y_t}$ be an arbitrary $x\in P_{1-Y_t}$ such that $f(X_\tau^0, X_\tau^1) = 1-Y_t$\;
        
        Let $\xi_\tau \gets ((X_\tau^0, X_\tau^1), Y_t$)\;
        
        Update $f(.,.) \gets \eta(\xi_1, \dots, \xi_\tau, (.,.))$\;
        
        Update $\tau \gets \tau + 1$\;
    }
    $\mathcal{V} \gets \mathcal{V}_{X_t, Y_t}^\U$\;
}
\end{algorithm}

We are now ready to state our main result in this section.

\begin{theorem}
\label{thm: realizable optimal mistake bound}
    For a hypothesis class $\H$ with $\ldimu(\H)=L < \infty$, the optimal mistake bound in the realizable robust online learning satisfies $\optmistake = L$.
    \begin{proof}
        We prove the theorem first by showing $\optmistake \geq L$ and then $\optmistake \leq L$.

        To prove the lower bound consider the following strategy by the adversary. Pick a tree of depth $L$ that is shattered by the class. Start from the root $(X_\emptyset^0, X_\emptyset^1)$ and pick any $Z_1 \in X_\emptyset^0 \cap X_\emptyset^1$. For any prediction $\yhat_1$ of the learner, pick $Y_1=1-\yhat_1$ and $X_1 = X_\emptyset^{Y_t}$. Go down along the edge labeled by $Y_t$ and continue the same process for the child there. By definition of the tree and $\ldimu$, the sequence picked by the adversary is realizable. Thus, we the adversary can force at least $L$ mistakes. Since the learner was arbitrary, $\optmistake \geq L$.

        We now proceed to prove the upper bound. We follow the prediction strategy outlined in Algorithm~\ref{alg: online learner}. Consider a round $t$ where the learner makes a mistake and consider $\xi_\tau=((x_0,x_1), Y_t)$ as defined in the algorithm. Then by definition of $f$, $\eta(\xi_1,\cdots, \xi_{\tau-1}, (x_0,x_1)) = \yhat_t \neq Y_t$. This means the orientation learner makes a mistake on all its inputs in a game defined by the sequence $\xi_1,\cdots, \xi_\tau$. Thus, the number of mistakes that the learner makes is at most the number of mistakes that the orientation learner makes. However, by Theorem~\ref{thm: optimal mistake orientation game}, we know there is a orientation learner, i.e., $\mathrm{SOA_{OG}}$, that makes at most $L$ mistakes. Thus, $\optmistake \leq L$. The only thing we need to address is that $x_{1-Y_t}$ exists. We know $x_{Y_t}=X_t \in P_{Y_t}$, thus it must be the case $\exists x_{1-Y_t}$ such that $f(x_0,x_1)=1-Y_t$ as otherwise the algorithm would have predicted $Y_t$ by the definition of $\yhat_t$.
    \end{proof}
\end{theorem}

\section{Multiclass robust online learning}
In the previous section, we only studied binary classes. In this section, we state similar results for a general label space $\Y$ potentially with an infinite size (e.g., $\Y=\mathbb{N})$. We use similar techniques to show our results. To do so, we define a multiclass orientation game where at each step $t$, in addition to presenting $(x_t^0, x_t^1)$, the adversary also presents two labels $y_t^0 \neq y_t^1$.

\begin{definition}[Multiclass Orientation Game]
    \label{def: multiclass orientation game}
    Orientation game is an iterative game between an adversary and a learner such that:
    
    At each round $t=1,2,\cdots$:
    \begin{enumerate}
        \item[] 
        \begin{enumerate}
            \item The adversary selects $(X_t^0, X_t^1) \in \X^2_\U$ and $Y_t^0, Y_t^1\in \Y$ s.t. $Y_t^0 \neq Y_t^1$ and reveals them to the learner.
            \item The learner predicts $\yhat_t \in \Y$.
            \item The adversary selects $i_t \in \{0,1\}$ and reveals $(X_t^{i_t}, Y_t^{i_t})$ to the learner.
            \item The learner incurs a loss of $\indict[\yhat_t \neq Y_t^{i_t}]$.
        \end{enumerate}
    \end{enumerate}
\end{definition}

We say the problem is realizable if the sequence $X_1^{i_1}, Y_1^{i_1}, \cdots, X_t^{i_t}, Y_t^{i_t}$ is realizable for each finite $t$. We now define a multiclass version of our dimension and again show that it characterizes learnability. 

A multiclass $\U$-adversarial Littlestone tree is a collection of nodes as follows:
$$\{(x^0_\mathbf{u}, x^1_\mathbf{u}, y_{\mathbf{u}}^0, y_{\mathbf{u}}^1)\in \X_\U^2 \times \Y_{\neq}^2:\mathbf{u}\in\{0,1\}^k, 0\leq k < d\},$$

where $\Y^2_{\neq} :=\{(y_1,y_2)\in \Y^2:y_1 \neq y_2\}$. We say such a tree is shattered by $\H$ if for each $\mathbf{u}\in \{0,1\}^d$ and for each $0 \leq k < d$, there exists $h\in\H$ such that for all $0 \leq i \leq k$, $h(z)=y_{\mathbf{u}_{\leq i}}^{u_{i+1}}$ for all $z\in \U(x_{\mathbf{u}_{\leq i}}^{u_{i+1}})$, i.e., $l_\U(h, (x_{\mathbf{u}_{\leq i}}^{u_{i+1}}, y_{\mathbf{u}_{\leq i}}^{u_{i+1}}))=0$. The multiclass $\U$-adversarial Littlestone dimension of $\H$ is similarly defined as the maximum depth of a tree that is shattered by $\H$. Here we abuse notation and also denote this dimension by $\ldimu(\H)$.

We now present our results for the multiclass learning problems. The proofs for these results are similar to those in section~\ref{sec: realizable binary}. Thus, we only give sketch of the proof and outline where they would differ from the binary case.
\begin{theorem}
    \label{thm: multiclass orientation}
    For any $\H$ with finite $L=\ldimu(\H)$ the optimal mistake bound achievable in the realizable multiclass orientation game equals $L$.
    \begin{proof}[proof sketch]
        The proof for lower bound is again by an adversary that plays according a shattered tree and for each prediction $\yhat_t=Y_t^y$ of the learner set the true label to $Y_t^{1-y}$ and follow the respective edge. The definition of the dimension then would ensure the adversary can continue for $\ldimu(\H)$ rounds.

        The upper bound is achieved by an adapted version of $\mathrm{SOA_{OG}}$ that predicts $\yhat_t = Y_t^y$ for $y\in \argmax_{y\in \{0,1\}} \ldimu(\mathcal{V}_t^{X_t^{y},Y_t^y})$ where $\mathcal{V}_t^{X_t^{y},Y_t^y}$ is the set of all hypotheses in the current version space that robustly label $X_t^{y}$ with $Y_t^y$. Again, the idea is to ensure every time the learner makes a mistake, the dimension of the version space reduces and since it will remain nonnegative, the number of times it decreases and thus the number of mistakes is bounded by $\ldimu(\H)$.
    \end{proof}
\end{theorem}

\begin{theorem}
    Any $\H \subseteq \Y^\X$ with $\ldimu(\H) = L < \infty$ is realizable robust online learnable with optimal mistake bound $L$.
    \begin{proof}
        The lower bound again simply follows from the definition of the dimension. The upper bound can be achieved by a modified version of the learner for the binary case. Here the learner predicts $\yhat_t=y$ if 
        $$\exists y\in \Y, \exists x_y\in P_y, {\forall y'\in \Y\backslash\{y\}}\forall x_{y'}\in P_{y'}: f(x_y,x_{y'}, y, y')=y.$$
        That is, the learner compares each candidate label against all other labels in the label space. Then by the same logic as the binary case, if the algorithm makes a mistake, then it must be that there is a $\xi=(x,x',y,y')$ that the orientation learner has made a mistake on. Thus, the number of mistakes is bounded by the number of mistakes an orientation learner makes which we can upper bound by $\ldimu(\H)$ by Theorem~\ref{thm: multiclass orientation}.
    \end{proof}
\end{theorem}

\section{Regret bounds for agnostic robust online learning}
    \label{sec: agnostic}

    In the agnostic setting, we do not assume the sequence is realizable by the class $\H$. However, instead of minimizing the number of mistakes, we minimize the regret of the algorithm (Equation~\ref{def: reg}). Moreover, we allow the learner to be randomized, in which case we are interested in the expected regret.

    To derive regret upper bounds in the agnostic setting, we use the technique of \citet{hanneke2023multiclass} where we compress the input sequence to the maximally realizable subsequence of the input sequence. 
    \begin{theorem}
    \label{thm: agn regret upper bound}
        The optimal expected regret on any sequence $S=Z_1, X_1, Y_1,\cdots, Z_T,X_T,Y_T$ is at most $\mathcal{O}(\sqrt{T\ldimu(\H)\log(T)})$.
        \begin{proof}
            For any $J\subseteq [T]$, let $S_J:=\{Z_t,X_t,Y_t:t\in J\}$ be the subsequence of $S$ on indices in $J$. Let $\A(z;S_J)$ be the realizable learner that has only been trained using the subsequence $S_J$ (since the sequence might not be realizable, we assume the learner predicts 0 if the version space becomes empty). Here we assume the realizable learner is lazy and only updates its history (version space) when it makes a mistake, which is enough for achieving the same upper bound in the realizable setting. This is because each mistake still reduces the dimension of the version space by at least one, so the number of mistakes will still be bounded by the dimension of the hypothesis class. Let $h^*\in \H$ be the hypothesis that achieves $\min_{h\in \H} \sum_{t=1}^T l_\U(h, (X_t,Y_t))$. Let $R=\{t\in[T]:l_\U(h^*, (X_t,Y_t))=0\}$. Then $S_R$ is realizable by $h^*$ and thus $\H$. Assume we play the robust online learning game on $S_R$, then by the guarantees of the (lazy) realizable learner, it makes $L\leq\ldimu(\H)$ mistakes. Let $J$ be the rounds on which the lazy learner makes a mistake. Consider the prediction strategy on the whole sequence that at each round $t$ predicts $\A_J(Z_t) = \A(Z_t;S_{J_{<t}})$ with $J_{<t}=\{j\in J: j<t\}$. Notice that this means on rounds in $[T]\backslash R$, we do not perform any updates on the learner, and on $R$ we only make updates when we make a mistake. Then by this definition, the mistakes on $S_R$ are exactly the subset $J$. Thus,
            \begin{align*}
            \sum_{t=1}^T 1[\A(Z_t;S_{J_{<t}}) \neq Y_t] &= \sum_{t\in R}1[\A(Z_t;S_{J_{<t}}) \neq Y_t] + \sum_{t\in [T]\backslash R}1[\A(Z_t;S_{J_{<t}}) \neq Y_t] \\
            &\leq L + T-|R|\\
            &\leq \ldimu(\H) + \sum_{t=1}^T l_\U(h^*, (X_t,Y_t))
            \end{align*}

            However, we do not know the sequence and thus $h^*, R$, and $J$ in advance, so we need to consider all possibilities of $J$, that is all subsets of size at most $\ldimu(\H)$ of $[T]$ whose count is at most $T^{\ldimu(\H)}$ for large enough $T$. As our final algorithm, $\A_{agn}$, we run the prediction with expert advice algorithm of Theorem~\ref{thm: pred with advice regret} over $\{A_J:J\subseteq T, |J|\leq \ldimu(\H)\}$, which gives us the following guarantee:

            \begin{align*}
            \sum_{t=1}^T \mathbb{E} [1[\A_{agn}(Z_t) \neq Y_t]] &\leq \min_{J\in \binom{[T]}{\leq \ldimu(\H)}} 1[\A_J(Z_t)\neq Y_t] + \mathcal{O}(\sqrt{T\log(T^{\ldimu(\H)})})\\
            &\leq \ldimu(\H) + \sum_{t=1}^T l_\U(h^*, (X_t,Y_t)) + \mathcal{O}(\sqrt{T\ldimu(\H)\log(T)})
            \end{align*}

            Thus, the expected regret is bounded by $\mathcal{O}(\sqrt{T\ldimu(\H)\log(T)})$.
        \end{proof}
    \end{theorem}

    \begin{lemma}
        For any class $\H$ and any agnostic learner, there exists a sequence on which the optimal expected regret is at least $\Omega(\sqrt{T\ldimu(\H)})$.
        \begin{proof}
            The proof follows similar ideas to those of the proof for classical setting by \citet{Ben-DavidPS09} (Lemma 14).
        \end{proof}
    \end{lemma}

\section{Robust online learning with uncertain perturbation sets}
\label{sec: uncertain perturbations}
In previous sections we assumed the learner knows $\U$. In this section we weaken this assumption so that the learner does not fully know $\U$ but has some prior knowledge about it. Formally, we assume the true perturbation function $\U^*$ belongs to some \textbf{finite} $\mathcal{G} \subseteq (2^\X)^\X$, which is known to the learner. Here we study this scenario for realizable learning in finite number of rounds $T$. The overall idea is to run the algorithms developed in the previous sections for all possible perturbation functions in $\mathcal{G}$ and then consider them as experts that we can use to make predictions. We are now ready to state our results.

\begin{theorem}
    \label{thm: uncertain, realizable 1}
    For any hypothesis class $\H$ with $L^*:=\max_{\U \in \mathcal{G}} \ldimu(\H) < \infty$ and for any input sequence that is $\U^*-$robustly realizable by $\H$ for some $\U^*\in\mathcal{G}$, the optimal expected number of mistakes is upper bounded by $L^* + \mathcal{O}(\sqrt{L^*\log(|\mathcal{G}|)} + \log(|\mathcal{G}|))$.
    \begin{proof}
        For each $\U\in \mathcal{G}$, define expert $\A_\U$ that predicts according to $\mathrm{SOA_{OG}}$ assuming $\U$ is the true perturbation function. It is possible that for some of the functions, the version space becomes empty. In that case, the expert will predict 1. By Theorem~\ref{thm: realizable optimal mistake bound}, we know $\A_{\U^*}$ will make at most $\mathrm{L}_{\U^*}(\H) \leq L^*$ mistakes. The results follow by using the algorithm for prediction with advice with $N=|\mathcal{G}|$ experts guaranteed by Theorem~\ref{thm: pred with advice known mistake bound}.
    \end{proof}
\end{theorem}

The mistake bound of Theorem~\ref{thm: uncertain, realizable 1} could be too loose in cases where $L_{\U^*}(\H)$ is much smaller than $L^*=\max_{\U\in \mathcal{G}}\ldimu(\H)$. In fact, $L^*$ could be infinite. Can we still get a bound on the number of mistakes? Our next result answers this question positively.

\begin{theorem}
    For any $\H$ and for any sequence that is $\U^*-$robustly realizable by $\H$ for some $\U^*\in \mathcal{G}$, the optimal number of mistakes is upper bounded by $(\mathrm{L}_{\U^*}(\H)+1)\log(|\mathcal{G}|)$.
    \begin{proof}
        Consider the experts $\A_\mathcal{U}$ as defined in the proof of Theorem~\ref{thm: uncertain, realizable 1}. The learner operates in multiple phases. In each phase the learner starts with the set of all experts and predicts with the majority label at each round. After each round the learner removes the experts that made a mistake. Each phase continues until the set of experts becomes empty, after which the next phase starts with all experts. In each phase, the learner makes at most $\log(|\mathcal{G}|)$ mistakes because each mistake reduces the number of experts by half. The expert $\A_{\U^*}$ makes at most $\mathrm{L}_{\U^*}$ mistakes, thus the number of phases is at most $\mathrm{L}_{\U^*}(\H)+1$ since in that phase there will be an expert that does not make a mistake and will not get removed. Thus, the total number of mistakes is bounded by $(\mathrm{L}_{\U^*}(\H)+1)\log(|\mathcal{G}|)$.
    \end{proof}
\end{theorem}

\section{Conclusion}
In this work, we initiated the study of robust online learning. We formulate this learning problem similar to the classic online learning framework of \citet{littlestone1988learning}. We also introduce a new dimension and show it characterize robust online learnability and proved mistake and regret bounds that are controlled by our dimension. Unlike the dimension that characterize robust PAC learnability \citep{MontasserHS22}, our dimension is simple and resembles the Littlestone dimension. We also studied a more general case where the learner only knows that the perturbation function belongs to a finite class. We prove upper bounds for both realizable and agnostic case that depend logarithmically on the cardinality of the class.

We conclude the paper with some questions for future work.
\begin{itemize}
    \item Here we assumed the learner either exactly knows $\U$ or knows it belongs to a finite class. What if the class that contains $\U$ is infinite, but has some structure? What about the case where we do not know anything about $\U$ but have access to some oracle? See for example \cite{montasser2021adversarially} and \cite{lechner2023adversarially} who studied this question in the PAC setting.
    \item In our setup, the learner receives the clean input $X_t$. Can we remove or weaken this assumption and still be able to learn?
    \item We also assume full feedback about the true label. What characterizes learnability in case of partial feedback (bandit setting)? See for example \cite{daniely2013price}, \cite{raman2024multiclass}, and \cite{raman2024apple}.
    \item Our lower and upper bound had a multiplicative difference of $\sqrt{\log(T)}$. Is it possible to close this gap? The same question has been answered for the classic online learning \citep{alon2021adversarial}.
    \item Here we focused on classification tasks. Can we extend our results to regression task? Robust PAC regression has been studied by \citet{pmlr-v202-attias23a}.
\end{itemize}

\printbibliography

\newpage
\appendix

\section{Prediction with expert advice}
\label{appendix: learning with expert advice}
Prediction with expert advice is an iterative interaction between a learner and a (potentially adversarial) environment. Here the learner has access to $N$ experts who will each at every round make a binary prediction. The learner the use this information to make a prediction of its own. The learner then incurs a loss based on the true label and will also receive the loss incurred by each expert. The goal of the learner is to minimize the total loss it incurs compared to the total loss that the best expert incurs in the hindsight.

\begin{theorem}[\cite{Cesa-Bianchi_Lugosi_2006}, Theorem 2.2]
    \label{thm: pred with advice regret}
    Let $N$ be the number of experts and let $T$ be the number of rounds. Let $l_t^i\in [0,1]$ be the loss that the expert $i$ incurs at round $t$. Then there is a random algorithm whose losses $l_t$, satisfy:
    $$\mathbb{E}[\sum_{t=1}^{T}l_t] - \min_{i\in [N]} \sum_{t=1}^{T}l_t^i \leq \mathcal{O}(\sqrt{T\log(N)})$$
\end{theorem}

\begin{theorem}[\cite{Cesa-Bianchi_Lugosi_2006}, Corollary 2.4]
    \label{thm: pred with advice known mistake bound}
    Consider the setting of Theorem~\ref{thm: pred with advice regret}. If we further assume that $\min_{i\in [N]} \sum_{t=1}^{T}l_t^i \leq L^*$, then there is a random algorithm whose losses $l_t$, satisfy:
    $$\mathbb{E}[\sum_{t=1}^{T}l_t] \leq L^* + \mathcal{O}(\sqrt{L^*\log(N)} + \log(N))$$
\end{theorem}

Note that the algorithms in the above theorems are the same algorithm with different values for a parameter. We refer the readers to \cite{Cesa-Bianchi_Lugosi_2006} for the algorithm and the proof as we only use the algorithm as a black box. However, what's crucial here is that this guarantee holds for any set of experts. Specifically, the experts do not need to be fixed and could adapt to the adversary, which is the case for us as the experts we create use the clean inputs in the previous rounds.

\end{document}